\documentclass[lettersize,journal]{IEEEtran}

\usepackage{amsmath,amssymb,amsfonts,amsthm,mathtools,wasysym}
\usepackage{algorithmic}
\usepackage{algorithm}
\usepackage{array}
\usepackage[caption=false,font=normalsize,labelfont=sf,textfont=sf]{subfig}
\usepackage{textcomp}
\usepackage{stfloats}
\usepackage{url}
\usepackage{verbatim}
\usepackage{graphicx}
\usepackage{cite}
\usepackage{csvsimple}
\usepackage[table]{xcolor}
\usepackage{listings}
\usepackage{soul}

\usepackage{comment}
\usepackage[hidelinks]{hyperref}
\usepackage{cleveref}
\Crefname{figure}{Fig.}{Figs.}

\usepackage{pgfplots}
\usepackage{tikz}
\usetikzlibrary{fit}
\usetikzlibrary{calc,positioning,shapes}
\usetikzlibrary{arrows.meta,positioning}
\usetikzlibrary{patterns,decorations.pathmorphing,decorations.markings}
\usetikzlibrary{svg.path}
\usetikzlibrary{backgrounds}

\pgfplotsset{compat=newest}

\usepackage[font=small,skip=2pt]{caption}

\definecolor{color0}{rgb}{0.12156862745098,0.466666666666667,0.705882352941177}
\definecolor{color1}{rgb}{1,0.498039215686275,0.0549019607843137}
\definecolor{color2}{rgb}{0.172549019607843,0.627450980392157,0.172549019607843}
\definecolor{color3}{rgb}{0.83921568627451,0.152941176470588,0.156862745098039}
\definecolor{color4}{rgb}{0.580392156862745,0.403921568627451,0.741176470588235}
\definecolor{color5}{rgb}{0,0,0}

\definecolor{ErrorPred}{rgb}{0.0196078431372549, 0.18823529411764706, 0.3803921568627451}
\definecolor{ErrorRef}{rgb}{0.2627450980392157, 0.5764705882352941, 0.7647058823529411}
\definecolor{ErrorPred2}{rgb}{0.403921568627451, 0.0002, 0.12156862745098039}
\definecolor{ErrorRef2}{rgb}{0.8392156862745098, 0.3764705882352941, 0.30196078431372547}

\newcommand{\lineWidth}{1.2pt}

\def\BibTeX{{\rm B\kern-.05em{\sc i\kern-.025em b}\kern-.08em
    T\kern-.1667em\lower.7ex\hbox{E}\kern-.125emX}}
    
\newcommand{\N}{\ensuremath\mathbb{N}}
\newcommand{\Z}{\ensuremath\mathbb{Z}}

\newcommand{\R}{\ensuremath\mathbb{R}}
\newcommand{\C}{\ensuremath\mathbb{C}}

\newcommand{\xone}{\ensuremath{x_1}}
\newcommand{\xtwo}{\ensuremath{x_2}}
\newcommand{\xonezero}{\ensuremath{x_{1,0}}}
\newcommand{\xtwozero}{\ensuremath{x_{2,0}}}
\newcommand{\wone}{\ensuremath{w_1}}
\newcommand{\wtwo}{\ensuremath{w_2}}

\newcommand{\Einit}{\ensuremath{\varepsilon_\mathrm{init}}}
\newcommand{\Eint}{\ensuremath{\varepsilon_\mathrm{int}}}
\newcommand{\Eeq}{\ensuremath{\varepsilon_\mathrm{eq}}}
\newcommand{\Eeqexp}{\ensuremath{\varepsilon_\mathrm{eq,\, exp}}}
\newcommand{\Ebc}{\ensuremath{\varepsilon_\mathrm{bc}}}
\newcommand{\Eref}{\ensuremath{\varepsilon_\mathrm{ref}}}
\newcommand{\newdelta}{\zeta}

\newcommand{\mexp}[0]{\mathrm{e}}

\newcommand{\dt}[0]{\partial_t}
\newcommand{\ds}{\,\mathrm{d}s}

\newcommand{\timeInt}{\mathbb{T}}

\newcommand{\tildeu}[0]{\mathbf{\tilde{u}}}
\newcommand{\hatu}[0]{\mathbf{\hat{u}}}
\newcommand{\uu}[0]{\mathbf{u}}
\newcommand{\tildee}[0]{\mathbf{\tilde{e}}}

\theoremstyle{plain}
\newtheorem{theorem}{Theorem}[section]
\newtheorem{proposition}[theorem]{Proposition}

\newtheorem{corollary}[theorem]{Corollary}
\newtheorem{remark}[theorem]{Remark}
\newtheorem{definition}[theorem]{Definition}
\newtheorem{assumption}[theorem]{Assumption}
\Crefname{assumption}{Assumption}{Assumptions}

\newtheorem{target}[theorem]{Problem}

\begin{document}

\title{Certified machine learning: Rigorous a posteriori error bounds for PDE defined PINNs
\thanks{B.~Hillebrecht acknowledges funding from the International Max Planck Research School for Intelligent Systems (IMPRS-IS). B.~Unger acknowledges funding from the DFG under Germany's Excellence Strategy -- EXC 2075 -- 390740016 and The Ministry of Science, Research and the Arts Baden-Württemberg, 7542.2-9-47.10/36/2. Both authors are thankful for support by the Stuttgart Center for Simulation Science (SimTech).\\
\IEEEauthorrefmark{1},\IEEEauthorrefmark{2}: \textit{Stuttgart Center for Simulation Science}, \textit{University of Stuttgart}, Stuttgart, Germany, \{birgit.hillebrecht,benjamin.unger\}@simtech.uni-stuttgart.de ;\\
\IEEEauthorrefmark{1} %birgit.hillebrecht@simtech.uni-stuttgart.de ;
ORCID: 0000-0001-5361-0505 
\IEEEauthorrefmark{2} %benjamin.unger@simtech.uni-stuttgart.de ; 
ORCID: 0000-0003-4272-1079
}
}

% peer-review
\author{\IEEEauthorblockN{Birgit Hillebrecht\IEEEauthorrefmark{1}}
and 
\IEEEauthorblockN{Benjamin Unger\IEEEauthorrefmark{2}}
}


\maketitle

\begin{abstract}
    Prediction error quantification in machine learning has been left out of most methodological investigations of neural networks, for both purely data-driven and physics-informed approaches. Beyond statistical investigations and generic results on the approximation capabilities of neural networks, we present a rigorous upper bound on the prediction error of physics-informed neural networks. This bound can be calculated without the knowledge of the true solution and only with a priori available information about the characteristics of the underlying dynamical system governed by a partial differential equation. We apply this a posteriori error bound exemplarily to four problems: the transport equation, the heat equation, the Navier-Stokes equation and the Klein-Gordon equation.   
\end{abstract}


\section{Introduction}

Physics-informed machine learning is applied to numerous highly complex problems, such as turbulence and climate modeling \cite{WuXP18,CruTSB19, XiaWWP17}, model predictive control \cite{ArnK21, NicKFU22}, and Hamiltonian system dynamics \cite{GreDY19}. The systematic study of physics-informed machine learning as a method, however, remains open \cite{MenSCGL22}. Part of this is the question about the quality of the forecast: how close is the prediction of the physics-informed neural network (PINN) \cite{RaiPK19} to the actual solution?

Beyond statistical evaluations, such as in \cite{JinLTK20}, first steps to answer this question were taken in \cite{HilU22} by establishing rigorous error bounds for PINNs approximating the solution of ordinary differential equations. We extend these results by deriving guaranteed upper bounds on the prediction error for problems modeled by linear partial differential equations (PDEs). Our methodology is thereby applicable to, e.g., the previously mentioned use cases of PINNs in computational science and engineering. Similar to \cite{HilU22}, the error bound can be computed a posteriori without knowing the actual solution. 

Since there are extensive studies on the convergence and approximation properties of (physics-informed) neural networks \cite{Bar94, HorSW90, DuaYLLQY21, ShiZK20}, it is clear from a theoretical point of view that a suitable neural network (NN) can be found with the desired accuracy. These a priori results and the associated error estimates \cite{ShiZK20, HonSX21, DeRM22} are fundamentally different compared to the computable a posteriori error bounds we present here. Moreover, we like to stress that our analysis does not rely on statistical approaches or the Bayesian perspective but provides a rigorous error bound, which serves as a certificate for the~NN.

Close in topic is \cite{ShiZK20} wherein a so called 'a posterior' error estimate is obtained based on norm inequalities. But, as stated in their conclusion, the estimates are not directly computable due to unknown constants. In distinction, the error estimator presented here is computable which we demonstrate with several academic examples. Other earlier works \cite{MinR21, BerCP22} also differ from the following work in fundamental aspects, such as that the error bounds derived in them are either not guaranteed to hold, are restricted to a certain problem type, or require discretization steps.

At this point, we would like to emphasize that the results presented in \Cref{sec::error_est} are independent of PINNs as the chosen methodology but can be applied to other surrogate modeling techniques. Nevertheless, since PINNs use the norm of the residual as a regularization term during the training process, they appear as a natural candidate to apply our residual-based error estimator.

\paragraph*{Notation}
We use bold notation for vectors and vector-valued functions and italic letters for operators, especially we use $\mathcal{I}$ to denote the identity operator. The notations $\dot{\mathbf{u}} = \partial_t \mathbf{u} = \frac{\partial \mathbf{u}}{\partial t}$ are used interchangeably to denote partial derivatives w.r.t. time. Similarily, we shorten derivatives w.r.t. spatial variables and use conventional notation for the divergence $\nabla \cdot \phi = \mathrm{div}(\phi)$, the gradient $\nabla \mathbf{u} = \mathrm{grad}(\mathbf{u})$, and the Laplace operator $\Delta \mathbf{u} = \mathrm{div}(\mathrm{grad}(\mathbf{u}))$ w.r.t. to spatial variables only.
For normed spaces $(X, \| \cdot \|_X)$, $(Y, \|\cdot \|_Y)$, and operator $\mathcal{A}\colon D(\mathcal{A}) \subseteq X \rightarrow Y$ we define the induced operator norm $\| \mathcal{A} \|_{X, Y} \coloneqq \mathrm{sup}_{x\in D(\mathcal{A}), \|x\|_X \neq 0} \frac{\|\mathcal{A}x\|_Y}{\|x\|_X}$. In case, $X \subseteq Y$ and $\| \cdot \|_Y = \| \cdot \|_X$, we shorten the notation to $\| \mathcal{A} \|_{X}$, and drop the subscript whenever the norm is clear from the context. For the spatial domain $\Omega \subseteq \R^d$, we denote the boundary by $\partial \Omega$ and define the size of $\Omega$ as $\|\Omega \| = \int_\Omega 1 \mathrm{d}\mathbf{x}$. The considerd temporal domain is denoted by $\timeInt = [0, t_f)$ with $t_f \in \R_+$.
We use the conventional notation for spaces of p-integrable functions $L^p(\Omega)$, and Sobolev spaces $W^{k,p}(\Omega)$ and $H^k(\Omega) = W^{k, 2}(\Omega)$.

\section{Problem description}

In the following, we aim at solving the boundary and inital value problem (BIVP): find $\mathbf{u}\colon \timeInt \times \Omega \rightarrow \R^n$ such that
\begin{equation}
	\label{eq::pde}
	\left\{\quad
    \begin{aligned}
        \dt \mathbf{u} &= \mathcal{A}\mathbf{u} \qquad &\mathrm{in } \quad  \timeInt \times \Omega , \\
        \mathbf{u} &= \mathbf{u_0}  \qquad& \mathrm{in } \quad  \{t=0\} \times \Omega , \\
        \mathcal{B} \mathbf{u} &= \mathbf{u_b} \qquad &\mathrm{in } \quad \timeInt \times \partial \Omega .
    \end{aligned}\right.
\end{equation}
The initial value $\mathbf{u_0}$ lies in the domain of definition of the linear differential operator $\mathcal{A}\colon D(\mathcal{A}) \subseteq X \rightarrow X$, meaning $\mathbf{u_0} \in D(\mathcal{A})$. 
The bounded linear operator $\mathcal{B}\colon D(\mathcal{A}) \rightarrow U$  represents the type of the boundary condition and maps to the Banach space $U$ which contains the boundary condition $\mathbf{u_b}$.
As an example, for the Dirichlet boundary condition $\mathcal{B} = \mathcal{T} \mathcal{I}$ is suitable, wherein $\mathcal{T}$ denotes the trace operator.

With this setup given, the objective can be intuitively formulated as follows.
\begin{target} \label{problemdescription::informal}
    Given an approximate solution $\hatu$ for the BIVP \eqref{eq::pde}, which is determined, for example, by approximating the system with a NN. Find a computable certificate $\varepsilon \colon \timeInt \rightarrow \R_+$ such that
    \begin{equation*}
        \| \hat{\mathbf{u}}(t, \cdot)- \mathbf{u}(t, \cdot) \|_X \le \varepsilon(t)
    \end{equation*}
    without knowing the true solution $\uu$. 
\end{target}

In the following, we consider the BIVP \eqref{eq::pde} in the framework of semigroups  and investigate mild solutions 
\begin{equation*} 
    \mathbf{u}(t, \cdot) = \mathcal{S}(t) \mathbf{u_0}(\cdot)
\end{equation*}
defined by the semigroup of operators $\{\mathcal{S}(t)\}_{t\ge 0}$ generated by~$\mathcal{A}$.
To assert that the problem is well-defined and that a unique (mild) solution exists, we make numerous assumptions.
Firstly, for $\mathcal{B}$ to be well-defined, we need to assert the existence of a trace operator $\mathcal{T}$ via trace theorems, e.g. requiring $\Omega$ to be a Lipschitz domain \cite{Din96} or $\partial \Omega$ to be $C^1$ \cite[Sec. 5.5, Thm. 1]{Eva10}. 
By formulating the following assumption, we also include generic boundary conditions \cite{Sch20}.

\begin{assumption} \label{assumption::semigroup}
    For a reflexive Banach space $X$, the linear (differential) operator $\mathcal{A}' = \mathcal{A}|_{\mathrm{ker} \mathcal{B}}$
    %\begin{equation*}
    %    \mathcal{A}' \colon D(\mathcal{A}') \rightarrow X 
    %\end{equation*}
    generates a strongly continuous one-parameter semigroup of bounded linear operators $\{\mathcal{S}(t)\}_{t\ge 0}$ on $X$. The domain $D(\mathcal{A'}) =D(\mathcal{A}) \cap \mathrm{ker}(\mathcal{B})$ is a linear subspace of X. The boundary operator $\mathcal{B}$ is right invertible with the inverse $\mathcal{B}_0$ and $\mathcal{A}' \mathcal{B}_0$ is a bounded linear operator from $U$ to $X$.
\end{assumption}

A mild solution can then be defined by an abstract variation-of-constants type formula with an extended solution space. To make the presentation accessible to a large audience, we omit the general functional analytic framework here (see \cite{Sch20} and the references therein for the details). Furthermore, we drop the prime on the operator $\mathcal{A}'$ to simplify the notation. 

\begin{remark} \label{rem::inhomogeneous}
    The results derived below can be generalized directly to systems modeled by inhomogeneous PDEs of type
    \begin{equation*}
        \partial_t \mathbf{u} = \mathcal{A} \mathbf{u} + \mathbf{f}
    \end{equation*} 
    with $\mathbf{f}(t) \in D(\mathcal{A})$ and $\mathbf{f}$ independent of $\mathbf{u}$. To simplify notation, we continue all investigations without $\mathbf{f}$.
\end{remark}

Similarly to \cite{HilU22} for a finite-dimensional setting, the NN approximation of the solution of the BIVP~\eqref{eq::pde} is interpreted as the solution of a perturbed problem. Here we consider two types of perturbed problems, the first has the same boundary condition as the unperturbed problem, but has perturbation terms in the initial condition and the differential equation. The perturbed problem reads
\begin{equation}\label{eq::pde_approx}
	\left\{\quad
    \begin{aligned}
        \dt \hatu &= \mathcal{A}\hatu + \mathbf{R} \qquad &\mathrm{in } \quad  \timeInt \times \Omega, \\
        \hatu &= \mathbf{u_0} + \mathbf{R_0} \qquad& \mathrm{in } \quad  \{t=0\} \times \Omega ,\\
        \mathcal{B}\hatu &= \mathbf{u_b} \qquad &\mathrm{in } \quad \timeInt \times \partial \Omega ,  
    \end{aligned}\right.
\end{equation}
with perturbances $\mathbf{R_0} \in D(\mathcal{A})$ and $\mathbf{R}:\; \timeInt \rightarrow X$. 
For an approximate solution $\hatu$ given by a NN, we can determine the  terms $\mathbf{R}, \; \mathbf{R_0}$ by computing the residuals
\begin{equation} \label{eq::residuals}
    \begin{aligned}
        \mathbf{R_0} &\coloneqq \hat{\mathbf{u}}(0, \cdot) -\mathbf{u_0} \quad &\text{in } &\Omega , \\
        \mathbf{R} &\coloneqq \dt \hatu - \mathcal{A}\hatu \quad &\text{in }& \timeInt \times \Omega .
    \end{aligned}
\end{equation}
This calculation requires $\hatu$ to be differentiable in time, which constrains the choice of the activation function of the NN. Also, we assume the following for the BIVP~\eqref{eq::pde_approx}.
\begin{assumption} \label{assumption::approx}
    The perturbed BIVP \eqref{eq::pde_approx} is such that $\mathbf{R}$ is Lipschitz continuous in $t\in \timeInt$ and $\mathbf{u_0}+\mathbf{R_0} \in D(\mathcal{A})$. 
\end{assumption}

Regarding the approximation by a NN, for most activation functions, the Lipschitz condition is satisfied, and only the second condition requires a more detailed investigation of the corresponding function spaces. 

The second problem involves additionally a perturbation in the boundary condition
\begin{equation}\label{eq::pde_approx_soft_bcs}
    \left\{\quad\begin{aligned}
        \dt \hatu &= \mathcal{A}\hatu + \mathbf{R} \qquad &\mathrm{in } \quad  \timeInt \times \Omega ,\\
        \hatu &= \mathbf{u_0} + \mathbf{R_0} \qquad& \mathrm{in } \quad  \{t=0\} \times \Omega ,\\
        \mathcal{B}\hatu &= \mathbf{u_b} + \mathbf{R_b}\qquad &\mathrm{in } \quad \timeInt \times \partial \Omega .
    \end{aligned}\right.
\end{equation}
We consider this system as a boundary control system (BCS, c.f.~\cite{TucW09}) and use the notion of input-to-state-stable (ISS) systems, see \cite{Son89}, to later derive rigorous error bounds.
\begin{definition}\label{def::iss_iiss}
    System~\eqref{eq::pde_approx_soft_bcs} is said to be ISS w.r.t.~the boundary term $\mathbf{u_b}+\mathbf{R_b}$, if there exist functions $\beta, \gamma$ and an operator $\mathcal{C}$ such that
    \begin{equation}\label{eq::iss}
        \begin{aligned}
            \| \mathbf{u}(t, \cdot) \| \le & \beta (\|\mathbf{u}(0, \cdot) \|, t) \\ &+ \gamma (\|\mathcal{C} (\mathbf{R_b}(s) + \mathbf{u_b}(s))\|_{L^\infty(0,t; \Omega)}) .
        \end{aligned}
    \end{equation}
    Here, $\gamma\colon \R_+ \rightarrow \R_+$ is in $\mathcal{K}$, which means that it is continuous and strictly increasing with $\gamma(0)=0$. Similarily, $\beta\colon \R_+ \, \times\, \R_+ \rightarrow \R_+$ is in $\mathcal{KL}$, meaning $\beta(\cdot, t) \in \mathcal{K}$ for all $t\ge 0$ and $\beta(s, \cdot)$ is continuous and strictly decreasing to 0 for all $s>0$ \cite{Son89}.
\end{definition}

In the following, we consider an ISS BCS, where we set $\mathbf{u_0} = 0$ and $\mathbf{u_b}=0$, such that~\eqref{eq::iss} simplifies to the second term. As an example, for the heat equation (see \Cref{subsec::heat}) we use $\gamma(s) = 1/3 s$. For the imposed Dirichlet boundary, the operator $\mathcal{C} = \delta_{\partial\Omega}$ is the Dirac Distribution on the boundary of the domain $\Omega$. 

\begin{remark}
    We present reasoning and theorems based on the notion of ISS, however, all results can be transferred easily to integral input-to-state-stable (iISS) systems \cite{Son89,JacNPS18}.
\end{remark}

\begin{assumption} \label{assumption::iss}
    The BCS \eqref{eq::pde_approx_soft_bcs} is ISS with respect to the boundary values $\mathbf{u_b}+\mathbf{R_b}$.
\end{assumption}

With these definitions and assumptions given, we can clarify the intuitive understanding of \Cref{problemdescription::informal} and present the following rigorous formulation.

\begin{target} \label{problemdescription::inX}
    Assume that the operators $\mathcal{A}$, $\mathcal{B}$ in the BIVP~\eqref{eq::pde} satisfy \Cref{assumption::semigroup} and let $\uu$ denote the unique mild solution of~\eqref{eq::pde}. Furthermore, let \Cref{assumption::approx,assumption::iss} hold true and let $\hatu$ denote the mild solution of \eqref{eq::pde_approx} or \eqref{eq::pde_approx_soft_bcs}. Find a certificate $\varepsilon(t)\colon \timeInt \rightarrow \R_+$ such that
    \begin{equation*}
    \| \hat{\mathbf{u}}(t, \cdot)- \mathbf{u}(t, \cdot) \|_X \le \varepsilon(t)
    \end{equation*}
    and can be computed a posteriori without knowledge of the true solution $\mathbf{u}$.
\end{target}

%\begin{target} \label{problemdescription::pointwise}
%In context of problem \ref{eq::pde} and the approximated problem \ref{eq::pde_approx} in combination with assumption \ref{assumption::semigroup}, find an upper limit of the absolute value of the error $\mathbf{e}(t, x) = \mathbf{u}(t, x) - \hat{\mathbf{u}}(t, x) $ at a specific point $x\in U$
%\begin{equation}\label{eq::pde_error_point}
%\|\mathbf{e}(t,x)\|_{\mathbb{R}^n} := \| \mathbf{u}(t,x) - \hat{\mathbf{u}}(t,x) \|_{\mathbb{R}^n}
%\end{equation}
%a posteriori with computationally low cost.
%\end{target}

\section{Error estimation} \label{sec::error_est}

We divide this section, which contains the main results of this work, into three subsections, whereof the first two are concerned with the two previously introduced perturbed BIVPs. Here, we adapt the terminology of hard and soft boundary constraints in neural networks \cite{YuLMK22}, meaning that hard boundary constraints denote the exact fulfilledness of the boundary conditions in BIVP \eqref{eq::pde_approx}, and soft boundary constraints correspond to the perturbed BIVP \eqref{eq::pde_approx_soft_bcs}. For the networks used to generate these approximate solutions, we require a network to be designed to fulfill the boundary condition exactly for hard boundary constraints or force the network to approximate the boundary condition by a suitable loss term for soft constraints. In the third subsection, we give an auxiliary result to use numerical integration methods to compute the error bounds.

\subsection{Hard boundary constraints}

A NN approximation with hard boundary constraints can be interpreted as solution to the perturbed BIVP~\eqref{eq::pde_approx} by defining the residuals as in~\eqref{eq::residuals}.

\begin{theorem}\label{thm::errorlimit}
    Let the BIVP \eqref{eq::pde} be given and suppose that \Cref{assumption::semigroup} is fulfilled. Furthermore, let $\hatu$ be  a mild solution to \eqref{eq::pde_approx} and presume that the terms $\mathbf{R}$ and $\mathbf{R_0}$ satisfy \Cref{assumption::approx}. Then, 
    \begin{equation} \label{eq::elim}
        \varepsilon(t) \coloneqq \| \mathcal{S}(t)\|_X  {\newdelta_0} + \int_0^t \|\mathcal{S}(t-s)\|_X {\newdelta}(s ) \,\mathrm{d}s ,
    \end{equation}
    wherein $\|\mathbf{R_0} \|_X \le \newdelta_0 \in \mathbb{R}_+$ and $\newdelta \colon \timeInt \rightarrow \mathbb{R}_+$, $\newdelta$ bounded and continuous with $\| \mathbf{R}(t) \|_X \le \newdelta(t)$ is a certificate for $\hatu$. 
\end{theorem}
\begin{proof}
    The dynamics of the error $\mathbf{e}(t,\mathbf{x}) \coloneqq  \hatu(t, \mathbf{x}) - \mathbf{u}(t, \mathbf{x})$ is governed by the BIVP 
    \begin{equation} \label{eq::pde_diff_system}
        \begin{aligned}
            \dt \mathbf{e} &= \mathcal{A}\mathbf{e}+ \mathbf{R} \qquad &\mathrm{in } \quad \timeInt \times \Omega,\\
            \mathbf{e} &= \mathbf{R_0} \qquad& \mathrm{in } \quad \{t=0\} \times \Omega , \\
            \mathcal{B}\mathbf{e} &= 0 \qquad &\mathrm{in } \quad \timeInt \times \partial \Omega ,
        \end{aligned}
    \end{equation}
    with $\mathbf{R}$, $\mathbf{R_0}$ as defined in \eqref{eq::residuals}.
    Using \Cref{assumption::semigroup}, the unique mild solution \cite[Ch.$\,$4, Cor.$\,$2.11.]{Paz89} of BIVP \eqref{eq::pde_diff_system} is given by
    \begin{equation*} 
        \mathbf{e}(t,\cdot) = \mathcal{S}(t)\mathbf{R_0}(\cdot) + \int _0^t \mathcal{S}(t-s)\mathbf{R}(s, \cdot) \mathrm{d}s.
    \end{equation*}

    Controlling the residuals from above by $\newdelta, \; \newdelta_0$, one can limit $\|\mathbf{e}(t, \cdot)\|_X$ from above by $\varepsilon$ as defined in \eqref{eq::elim}.
\end{proof}

\begin{corollary}\label{lem::errorlimit_strongly}
	Let all prerequisites for \Cref{thm::errorlimit} be fulfilled. Then there exist constants $M\geq 1$, $\omega\in\R$ such that
	\begin{equation} \label{eq::eps_lim}
		\varepsilon(t) \le \tilde{\varepsilon}(t) \vcentcolon= M \mathrm{e}^{\omega t} {\newdelta_0} + \int_0^t M \mathrm{e}^{\omega (t-s)} {\newdelta}(s )\ds .
	\end{equation}
	If furthermore $\{S(t)\}_{t\ge 0}$ is a
	\begin{itemize}
		\item $\omega$-contraction semigroup, then $M=1$;
		\item exponentially decaying semigroup, then $\omega < 0$.
	\end{itemize}
\end{corollary}

\begin{proof}
    By applying general results on strongly continuous semigroups \cite[Ch. 1, Thm. 2.2]{Paz89} to \eqref{eq::elim}, the existence of $M \ge 1$ and $\omega \ge 0$ for \eqref{eq::eps_lim} is shown.
    An $\omega$-contraction semigroup is characterized by $M=1$. By definition of an exponentially decaying semigroup, there exist constants $M \ge 1$ and $\mu > 0$ such that $\|\mathcal{S}(t)\|_X \le M\mexp^{-\mu t}$.
\end{proof}

Although \Cref{lem::errorlimit_strongly} only establishes the existence of $M$ and $\omega$, we emphasize that these constants can be computed in several applications, see \Cref{sec::NumericalResults}. For example, we can use the spectral mapping theorem (e.g., \cite[Ch. 2, Thm. 2.4]{Paz89}) to determine the growth bound $\omega$ to be the real part of the largest eigenvalue of the infinitesimal generator of the semigroup.

\subsection{Soft boundary constraints}

Enforcing boundary values with soft boundary constraints introduces an additional perturbation in the boundary, which is reflected in the perturbed BIVP~\eqref{eq::pde_approx_soft_bcs}. 

\begin{theorem}\label{thm::errorlimit_iss}
    Let the BIVP \eqref{eq::pde} fulfill \Cref{assumption::semigroup} and assume that it is ISS according to \Cref{assumption::iss}. Let the solution $\hatu$ to the perturbed problem \eqref{eq::pde_approx_soft_bcs} satisfy \Cref{assumption::approx}, then
    \begin{equation*} 
        \begin{aligned}
       	\varepsilon(t) \coloneqq& \; M \mexp^{\omega t }{\newdelta_0} + \int_0^t M \mexp^{\omega (t -s)} {\newdelta}(s ) \,\mathrm{d}s  \\ 
        &\qquad+\; \gamma ( \|\mathcal{C}\mathbf{R_b}\|_{L^\infty(0,t;\Omega)}),
        \end{aligned}
    \end{equation*}
    with $\|\mathbf{R_0} \|_X \le \newdelta_0 \in \mathbb{R}_+$ and $\newdelta \colon \timeInt \rightarrow \mathbb{R}_+$, $\newdelta$ bounded and continuous with $\| \mathbf{R}(t) \|_X \le \newdelta(t)$ is a certificate for $\hatu$.
\end{theorem}
\begin{proof}
    We introduce $\tildeu$ as a mild solution to \eqref{eq::pde_approx} with the same perturbations $\mathbf{R_0}, \mathbf{R}(t)$ considered in the BIVP \eqref{eq::pde_approx_soft_bcs}. Since a mild solution $\uu$ to \eqref{eq::pde} exists, and \Cref{assumption::approx} is valid, there exists a unique mild solution to \eqref{eq::pde_approx}. Hence, the error can be split into
    \begin{equation*}
        \| \hatu - \uu \|_X \le \| \hatu - \tildeu \|_X + \| \tildeu - \uu\|_X .
    \end{equation*}
    According to \Cref{lem::errorlimit_strongly} we can bound the latter part by~\eqref{eq::eps_lim}. 
    For the first contribution, we proceed similarly to the proof of~\Cref{thm::errorlimit} by observing that the dynamics of the error $\tildee(t,\mathbf{x}) \coloneqq \hatu - \tildeu $ is determined by the BIVP
    \begin{equation*} 
        \begin{aligned}
            \dt \tildee &= \mathcal{A}\tildee \qquad &\mathrm{in } \quad \timeInt \times \Omega,\\
            \tildee &= 0 \qquad& \mathrm{in } \quad \{t=0\} \times \Omega  ,\\
            \mathcal{B}\tildee &= \mathbf{R_b}\qquad &\mathrm{in } \quad \timeInt \times \partial \Omega .
        \end{aligned}
    \end{equation*}
    Now, we use the ISS property according to \Cref{assumption::iss} and notice that $\tildee(0, \cdot) = 0$ implies $ \beta(\| \tildee(0, \cdot) \|, t) =0$ (with~$\beta$ as in \Cref{def::iss_iiss}) and hence conclude
    \begin{equation*}
        \| \tildee(t, \cdot) \|_X \le \gamma (\|\mathcal{C}\mathbf{R_b} \|_{L^\infty(0, t; \Omega)}).
    \end{equation*}
     Adding the contributions to the error concludes the proof.
\end{proof}

\subsection{Integral approximation}

As in \cite{HilU22}, we give an auxiliary result to apply the previous theorems using numerical integration techniques to compute the integrals in \Cref{thm::errorlimit}, \Cref{thm::errorlimit_iss} and \Cref{lem::errorlimit_strongly}, respectively. Here, we use the composite trapezoidal rule to approximate the integral as
\begin{equation*}
    \int_0^t M\mexp^{\omega (t-s)} \newdelta(s) \ds \approx M \mexp^{\omega t} \hat{I}_n(0,t,\mathrm{e}^{-\omega s}\newdelta(s))
\end{equation*}
using $n$ equally spaced subintervals $[s_i, s_{i+1}]$, $0 = s_1< ...< s_n= t$, meaning $s_{i+1}-s_{i} = \Delta s$. Therein the right hand side is defined via
\begin{equation}
    \begin{aligned} \label{eq::trapz}
        \hat{I}_n&(t, \mexp^{-\omega s } \newdelta(s) ) \coloneqq \\ 
        &\sum_{i=1}^{n-1} \frac{\Delta s}{2} \left( \mexp^{-\omega s_i} \newdelta(s_{i}) + \mexp^{-\omega s_{i+1}} \newdelta(s_{i+1}) \right). 
    \end{aligned}
\end{equation}
Numerical computation of the integral introduces an additional error 
\begin{equation} \label{eq::trapz_err}
    \Eint \coloneqq M \mexp^{\omega t} \frac{K t^3}{12 n^2},
\end{equation}
assuming that $\newdelta(s)$ is sufficiently smooth and that $K$ fulfills
\begin{equation} \label{eq::K}
    \left|\tfrac{\mathrm{d}^2}{\mathrm{d} s^2} (\mexp^{-\omega s} \newdelta(s))\right| \le K .
\end{equation}
Summarizing the previous discussion yields the following result.

\begin{proposition}
	Consider the BIVPs~\eqref{eq::pde} and~\eqref{eq::pde_approx} (or \eqref{eq::pde_approx_soft_bcs}) and suppose that \Cref{assumption::semigroup,assumption::approx} (and \ref{assumption::iss}) are satisfied. Then
    \begin{equation*} 
        \varepsilon(t) \le \;  M \mexp^{\omega t} \newdelta_0 + M \mexp^{\omega t}\hat{I}_{n}(t, M \mexp^{-\omega s}  \newdelta(s)) + \Eint + \Ebc
    \end{equation*}
    with $\newdelta_0, \,\,\newdelta$ as in \Cref{thm::errorlimit} under the additional assumption that $\newdelta\in \mathrm{C}^3(\timeInt, \mathbb{R})$. Furthermore the numerical integral \eqref{eq::trapz} and the error of the numerical integration \eqref{eq::trapz_err} is used assuming $K$ can be found according to \eqref{eq::K}. When considering the BIVP~\eqref{eq::pde_approx_soft_bcs}, $\Ebc = \gamma (\|\mathcal{C}\mathbf{R_b} \|_{L^\infty(0, t; \Omega)})$, else the term vanishes.
    \label{lemma::errorlimit_trpz}
\end{proposition}

\begin{remark}
    The additional integration error considered in \Cref*{lemma::errorlimit_trpz} accounts for numerical integration over time only. The numerical error introduced by integrating over space is already included and part of $\newdelta(t)$ and $\newdelta_0$.
\end{remark}
\begin{remark}
    The previous derivations also show that an additional regularization of the derivatives (c.f.~\cite{YuLMK22}) of the residuum could prove useful to reduce the integration error.
\end{remark}

\section{Methodology: PINNs} \label{sec::methodology}

We want to apply the results from the previous \Cref{sec::error_est} to PINNs as introduced in \cite{RaiPK19}. Therein the the data-driven loss used during training of feed-forward NNs is supplemented by a contribution reflecting the PDE (the physics-informed contribution), yielding 
\begin{equation} \label{eq::loss}
    L = \tfrac{1}{1+\kappa} L_\mathrm{eq} +\tfrac{\kappa}{1+\kappa} L_\mathrm{data}  + \rho L_\mathrm{space} .
\end{equation}
The parameter $\kappa \in \R_+$ defines the relation between data-driven and physics-informed contributions, while $\rho\geq 0$ scales the contribution of soft boundary conditions and further spatial properties, see below for further details. The data-driven contribution is encoded in 
\begin{equation*}
    L_\mathrm{data} = \frac{1}{N_\mathrm{data}} \sum_{i=1}^{N_\mathrm{data}} \|  \hatu(t_{\mathrm{data}, i}, \mathbf{x}_{\mathrm{data}, i}) - \uu_{\mathrm{data}, i}\|^2
\end{equation*}
for a given training dataset 
\begin{equation*} (t_{\mathrm{data}, i}, \mathbf{x}_{\mathrm{data}, i}, \uu_{\mathrm{data}, i})_{i=1,...,N_\mathrm{data}}.
\end{equation*}
We need to know the true solution  $\uu_{\mathrm{data}, i} = \uu(t_{\mathrm{data}, i}, \mathbf{x}_{\mathrm{data}, i})$ for this and hence we reduce data-driven contributions to data given by the initial condition
\begin{equation*} 
    L_\mathrm{data} = \frac{1}{N_\mathrm{data}} \sum_{i=1}^{N_\mathrm{data}} \|  \hatu(0, \mathbf{x}_{\mathrm{data}, i}) - \mathbf{u_0}( \mathbf{x}_{\mathrm{data}, i})\|^2 .
\end{equation*}

The physics-informed loss  
\begin{equation*}
    L_\mathrm{eq} = \frac{1}{N_\mathrm{eq}} \sum_{i=1}^{N_\mathrm{eq}} \| \partial_t \hatu(t_{\mathrm{eq} , i}, \mathbf{x}_{\mathrm{eq} , i}) - \mathcal{A}\hatu(t_{\mathrm{eq} , i}, \mathbf{x}_{\mathrm{eq} , i})\|^2
\end{equation*}
evaluates the residual of the PDE for a set of $N_\mathrm{eq}$ collocation points $(t_{\mathrm{eq} , i}, \mathbf{x}_{\mathrm{eq} , i})_{i=1,..., N_\mathrm{eq}}$. In particular, this contribution to the loss term aims at minimizing $\mathbf{R}$, which is used in the error bounds derived in \cref{sec::error_est}.

The third contribution $L_\mathrm{space}$ accounts for space constraints such as the divergence-freeness of solutions to the Navier-Stokes equations or soft boundary constraints. The parameter $\rho$ scales this loss relative to the other two contributions.

We use NNs with $n_\mathrm{i}\in \mathbb{N}$ input neurons, $n_\mathrm{o}\in \mathbb{N}$ output neurons and $n_\mathrm{hl}\in \mathbb{N}$ hidden layers with $n_\mathrm{hn}\in \mathbb{N}$ neurons each as basis for networks used in the numerical experiments. In the following two subsections, we discuss the modifications in network topology or loss functions which are necessary to realize hard and soft boundary constraints.

\begin{remark}
    In case no further control parameters are given, $n_\mathrm{i}$ corresponds to the dimension of $\timeInt \times \Omega$, i.e., $n_\mathrm{i} = d+1$. Similarly, assuming that no hidden variables are included in the PDE which are not described explicitly in the PDE (like the pressure in the Navier-Stokes equations, c.f. \Cref{subsec::NSE}), $n_\mathrm{o} $ corresponds to the dimension of $\uu$, i.e., $n_\mathrm{o} = n$.
\end{remark}

\subsection{Hard boundary constraints}
\label{subsec::hard_bc}

The implementation of hard boundary constraints requires a modification of the network topology. For the following numerical experiments, we need topologies to account for Dirichlet or periodic boundary conditions. For Dirichlet boundary conditions, we extend the purely sequential core NN by a second path as depicted in \Cref{fig::model_dirichlet_bc}. In the homogeneous case, we use a mask (as presented for arbitrary boundary conditions in \cite{McF06,McFM09}  and therein called length factor) for a d-dimensional rectangular domain of width $w = (w_0,..., w_d)$ with center $x_0 = (x_{1,0}, ..., x_{d, 0})$ of the type
\begin{equation*}
    \lambda(\mathbf{x}) = \prod_{i=1}^d \left[ (\tfrac{w_i}{2})^2 - (x_i - x_{i,0})^2 \right].
\end{equation*}
For inhomogeneous problems, we add a suitable offset 
\begin{equation*}
    \boldsymbol{\chi}(t, \mathbf{x}) = \mathbf{u_b}(t, \mathbf{x}) \qquad \mathrm{for }\;\; t\in\timeInt,\quad \mathbf{x}\in \partial \Omega.
\end{equation*} 
We can summarize the model to be 
\begin{equation} \label{eq::hard_bcs}
    \hatu(t, x)  = \lambda(\mathbf{x}) \cdot \mathbf{\hat{\hat{u}}}(t, \mathbf{x}) + \boldsymbol{\chi}(t, \mathbf{x})
\end{equation}
wherein $\mathbf{\hat{\hat{u}}}$ is the output of the sequential core NN.

For periodic boundary condition, we consider again a d-dimensional rectangular domain of width $w = (w_0,..., w_d)$ and follow \cite{DonN21} to construct a periodic base layer for one (the $i$-th) space dimension with~$n_\mathrm{p}$ units as
\begin{equation} \label{eq::per_base}
    \sigma ( \mathbf{A} \cdot \cos(\boldsymbol{\nu} x + \boldsymbol{\phi}) + \mathbf{B}),
\end{equation}
wherein $\sigma$ is the activation function. Furthermore, $\mathbf{A}, \mathbf{B}, \boldsymbol{\nu}, \boldsymbol{\phi} \in \R^{n_p}$. The frequency $\nu$ may not be trained but needs to be fixed by the period $ w_i$ in the respective space direction as $\boldsymbol{\nu} =2 \pi / w_i  \cdot \mathbf{I}_{n_p}$. Here, $\mathbf{I}_{n_p}$ denotes the vector of all ones in $\R^{n_p}$. 

The non-periodic input variables are processed by a dense layer.
Exemplarily, a non-sequential network realizing this is shown in \Cref{fig::model_periodic_bc}. To enhance training performance, we introduce a small dense network before the periodic layers to model the dependence of $\mathbf{A}^i, \mathbf{B}^i, \boldsymbol{\phi}^i$ on $t$ as an extension to the suggested modelling in \cite{DonN21}. 

\begin{figure} 
    \centering
    \scalebox{0.7}{
        \begin{tikzpicture}[ node distance = 2cm, >=latex,block/.style={draw, fill=white, rectangle, minimum height=3em, minimum width=8em}]
            \node[block] (Input) {Input};
            \begin{scope}[node distance=10mm]
                \node[block, below =of Input] (NN){Hidden layers};
                \node[block, below =of NN] (Output){NN Output};
            \end{scope}
            \node[block, right =of NN] (LF) {Mask};

            \begin{scope}[node distance=28mm]
                \node[block, below =of LF] (RO) {$\times$};
            \end{scope}
            \node[block, right =of LF] (OFS) {Offset};

            \begin{scope}[node distance=46mm]
                \node[block, below =of OFS] (ROO) {$+$};
            \end{scope}
    
            \draw[<-] (Input) -- ++ (0,1.5) node[right] {$\mathbf{x}$};
            \draw[<-, dashed] (Input.120) -- ++ (0,1) node[left] {$t$};

            \draw[->] (Input) -- (NN); 
            \draw[->, dashed] (Input.240) -- (NN.120); 
            \draw[->, dash dot dot] (NN) -- ++ (0,-1)  -| node[pos=0.25] {} (Output);
            \draw[->,dash dot dot] (Output) -- ++ (0,-0.8)  -| node[pos=0.05, below] {$\mathbf{\hat{\hat{u}}}(t,\mathbf{x})$} (RO.140); 
        
            \draw[->] (Input.320) -- ++ (0,-0.5)  -| node[pos=0.25, below] {$\mathbf{x}$} (LF); 
            \draw[->] (Input.320) -- ++ (0,-0.5)  -| node[above] {} (OFS); 
            \draw[->, dashed] (Input.330) -- ++ (0,-0.2)  -| node[pos=0.2, above] {$t$}  (OFS.50); 
            \draw[->, dash dot dot] (LF) -- ++ (0,-1)  -| node[pos=0.7, right] {$\lambda(\mathbf{x})$} (RO); 
            \draw[->,dash dot dot] (OFS) -- ++ (0,-1)  -| node[pos=0.7, right] {$\boldsymbol{\chi}(t,\mathbf{x})$} (ROO); 

            \draw[->,dash dot dot] (RO) -- ++ (0,-0.8) -| node[pos=0.2, below] {$\lambda(\mathbf{x}) \cdot \mathbf{\hat{\hat{u}}}(t, \mathbf{x})$} (ROO.140);

            \draw[->,dash dot dot] (ROO) -- ++ (0,-1) node[left] {$\hatu(t, \mathbf{x}) = \lambda(\mathbf{x}) \cdot \mathbf{\hat{\hat{u}}}(t, \mathbf{x}) + \boldsymbol{\chi}(t, \mathbf{x})$};

        \end{tikzpicture}
    }    
    \caption{NN realizing hard Dirichlet boundary constraints $\mathbf{u_b}(t, \mathbf{x})$ for $\mathbf{x} \in \partial \Omega$, $t\in \timeInt$. Input parameters are realized as dashed (time) and continuous (space) lines, processed intermediate or output values are shown as dash dotted lines.}
    \label{fig::model_dirichlet_bc}
\end{figure}
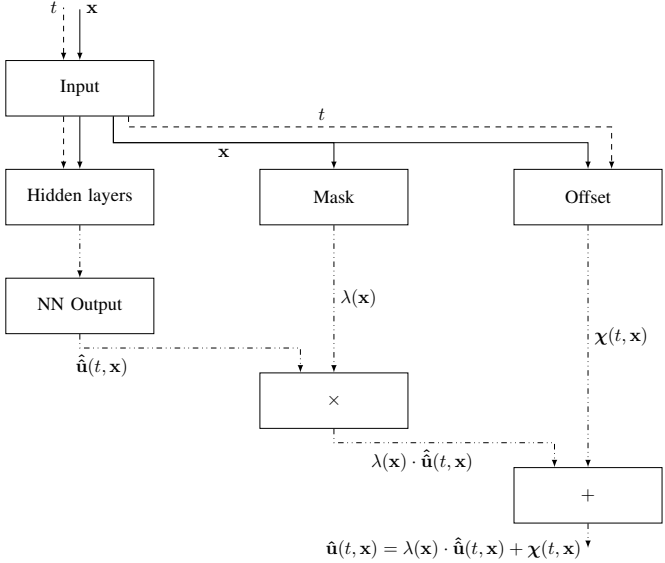

\begin{figure} 
    \centering
    \scalebox{0.7}{%
    \begin{tikzpicture}[node distance=10mm and 4mm,>=latex,block/.style={draw, fill=white, rectangle, minimum height=3em, minimum width=6em}]
            % input
            \node[block] (Input) {Input};
            \draw[<-] (Input.40) -- ++ (0,1) node[left] {$t$};
            \draw[<-, dashed] (Input.90) -- ++ (0,1) node[left] {$x_1$};
            \draw[<-, dotted] (Input.140) -- ++ (0,1) node[left] {$x_2$};
            
            % periodic block
            \node[block, below =of Input] (A){$\mathbf{A}$};
            \node[block, left =of A] (B){$\mathbf{B}$};
            \node[block, right =of A] (phi){$\boldsymbol{\phi}$};
            \node[block, below =of B] (per){Periodic base layer};
            \node[draw, black!50 , dashed,fit=(A) (B) (phi) (per)] (a12) {};
            \node[right =of  per, black!50, xshift=0.5cm, align=left] (per_all) {Enforces periodic\\  boundary condition};
            \draw[->] (Input) -- ++ (0,-1)  -| node[pos=0.25] {} (phi); 
            \draw[->] (Input) -- ++ (0,-1)  -| node[pos=0.25] {} (A); 
            \draw[->] (Input) -- ++ (0,-1)  -| node[pos=0.25, below] {$t$} (B); 
            \draw[->, dotted] (Input) -- ++ (-1,0)  -| node[above] {$x_2$} (per.155); 
            \draw[->, dash dot dot] (A) -- ++ (0,-1)  -| node[pos=0.25] {} (per.50); 
            \draw[->, dash dot dot] (B) -- ++ (0,-1)  -| node[pos=0.25] {} (per.50); 
            \draw[->, dash dot dot] (phi) -- ++ (0,-1)  -| node[pos=0.25] {} (per.50); 

            % linear x_1
            \node[block, below right =of phi, xshift=1cm] (lin){Dense layer};
            \draw[->] (Input) -- ++ (1,0)  -| node[pos=0.25, above] {$t$} (lin); 
            \draw[->, dashed] (Input.345) -- ++ (1,0)  -| node[pos=0.25, below] {$x_1$} (lin.120); 

            % common hidden layers
            \node[block, below right =of per] (concat){Hidden layers};
    
            % rest of network
            \draw[->, dash dot dot] (concat) -- ++ (0,-1) node[left] {...};
            \draw[->, dash dot dot] (per) -- ++ (0,-1)  -| node[pos=0.25] {} (concat); 
            \draw[->, dash dot dot] (lin) -- ++ (0,-1)  -| node[pos=0.25] {} (concat); 
    \end{tikzpicture}
    }
    \caption{Extended realization of a non sequential network to enforce the periodic boundary condition in the space variable $x_2$ according to \eqref{eq::per_base}.
    The periodic base layer is dependent on three parameters $\mathbf{B}, \mathbf{A}$ and $\phi$, which have an explicit dependency on the time $t$. These parameters and the input variable for which the output shall be periodic, $x_2$, serve then as input to the periodic base layer.
    The other input variables $x_1$ and $t$ are bypassed around the additional elements through a dense layer.}
    \label{fig::model_periodic_bc}
\end{figure}
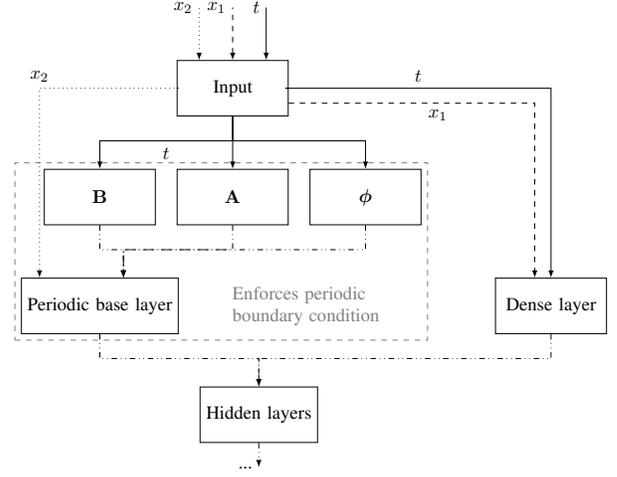

\subsection{Soft boundary constraints and other restrictions}

In case, the solution shall meet additional constraints which are not satisfied automatically by the NN, e.g. divergence-freeness, we train the NN by adding $L_\mathrm{space}$ and weighting it with the parameter $\rho$.
Details for the choice of $\rho$ and the implementation of $L_\mathrm{space}$ can be found in the corresponding examples in \Cref*{subsec::NSE}, \ref{subsec::KG} and \ref{subsec::heat}.

If soft boundary constraints are used, the boundary error needs to be taken into account according to \Cref{thm::errorlimit_iss}.

\section{Numerical results}  \label{sec::NumericalResults}

In this section, we investigate four examples and illustrate the applicability of the a posteriori error estimator for PDE-defined PINNs. The four examples are the heat equation, the transport equation, the Navier-Stokes equation, and the Klein-Gordon equation.

\vspace{0.2cm}
\noindent\fbox{%
    \parbox{0.48\textwidth}{%
        \small The code and data used to generate the subsequent results are accessible via doi: 10.5281/zenodo.7156168 under MIT Common License.
    }%
}
\vspace{0.2cm}

Herein, we apply \Cref{lem::errorlimit_strongly} and illustrate the different contributions by splitting the predicted error in a part reflecting the propagation of the error in the initial condition $\Einit$ and the integrated error from the misrepresentation of the temporal evolution $\Eeq$ according to 
\begin{equation*}
    \begin{aligned}
        \Einit (t) \coloneqq &\; M \mathrm{e}^{\omega t} \newdelta_0, \\
        \Eeq (t) \coloneqq & \;\int_0^t M \mathrm{e}^{\omega (t-s)}  \newdelta(s )\ds .
    \end{aligned}
\end{equation*}
We omit the numerical approximation steps as derived in \Cref{lemma::errorlimit_trpz} for notation simplicity. Furthermore, if we apply \Cref{thm::errorlimit_iss}, we denote the error introduced by the boundary approximation as 
\begin{equation*}
    \Ebc (t) \coloneqq \;\gamma(\|\mathcal{C}\mathbf{R_b}\|_{L^\infty(0,t; \Omega)} ).
\end{equation*}
The reference error is denoted by 
\begin{equation*}
   \Eref (t) \coloneqq \; \sqrt{ \int_\Omega \|\hatu (t, \mathbf{x})- \uu (t, \mathbf{x})\|^2\, \mathrm{d}\mathbf{x}}, 
\end{equation*}
which is computed numerically with composite trapezoidal rule. To apply \Cref{lemma::errorlimit_trpz}, we estimate the number of required subintervals by requiring  
\begin{equation*}
    \Eint \le \alpha \Eeqexp   ,
\end{equation*}
wherein $\Eeqexp$ is an initial estimation of the error $\Eeq$. We compute this estimation as 
\begin{equation*}
    \Eeqexp(t) = M \exp ^ {\omega t} \newdelta_0 + M (\mexp^{\omega t} - 1)\frac{\overline{\newdelta} \cdot \| \Omega \| }{\omega} 
\end{equation*}
using the average residual during training 
\begin{equation*}
    \overline{\newdelta} = \frac{1}{N_\mathrm{eq}} \sum_{i=1}^{N_\mathrm{eq}} \| \partial_t \hatu (t_{\mathrm{eq},i}, \mathbf{x}_{\mathrm{eq},i}) - \mathcal{A}\hatu (t_{\mathrm{eq},i}, \mathbf{x}_{\mathrm{eq},i}) \|_{2, \R^n} .
\end{equation*}
The parameter $\alpha$ scales the numerical integration error relative to the expected error $\Eeqexp$ and is chosen constantly as $\alpha = 0.33$ in the following experiments.
Hence, we can compute the number of required subintervals  
\begin{equation*}
    N_\mathrm{SI} (t) = \left\lceil \sqrt{\frac{\mexp^{\omega t} K \, t^3}{12 \alpha \Eeqexp(t)}} \,\right\rceil .
\end{equation*}

We setup NNs based on TensorFlow \cite{Tensorflow15} using the hyperbolic tangent as the activation function. Problem specific parameters such as network size are mentioned in the relevant subsections. A set of common parameters and settings is listed here: We train using the L-BFGS optimizer \cite{LiuN89} with learning rate $0.1$ unless stated otherwise. The NNs rely on the hyperbolic tangent as activation function and the data-driven and physics-informed contributions are equally weighted with $\kappa = 1$ unless explicitly mentioned.

For testing the NN and our error estimator, we generate a grid of equally spaced data points in the complete spatial domain and evaluate the prediction of the NN and the error estimator at various points in time.

\subsection{Heat equation} \label{subsec::heat}

We consider the scalar BIVP
\begin{equation}\label{eq::heat}
    \begin{aligned}
        \partial_t u(t,x) &= \tfrac{1}{5} \partial_{xx} u(t,x),& \\
        u(0, x) &= \sin(2\pi x ) \qquad &\mathrm{for }&\; x\in \Omega, \\
        u(t, 0) &= 0 = u(t, 1) \qquad &\mathrm{for }&\; t \in \mathbb{T} .
    \end{aligned}    
\end{equation}
for the spatio-temporal domain $\timeInt \times \Omega = [0,0.5] \times (0,1)$ with analytical solution $u(t, x) = \sin(2 \pi x) \mexp^{-\frac{1}{5} (2\pi)^2 t}$; see~\cite{Fou78} .

\paragraph{Semigroup properties}
The linear operator under consideration is defined by
\begin{equation*}
    \begin{aligned}
        \mathcal{A}\colon D(\mathcal{A}) = H_0^2(\Omega) \to L^2(\Omega),\qquad 
        u \mapsto \tfrac{1}{5} \partial_{xx} u.
    \end{aligned}
\end{equation*}

The operator $\mathcal{A}$ generates a $0 $-contraction semigroup on $L^2(\Omega)$ \cite[Ch. 7.4.3, Thm. 5]{Eva10}, which is also exponentially stable on $L^2(\Omega)$ using results from \cite{SeiTW22}.
\begin{proposition} \label{prop::heat_exp}
    The operator $\mathcal{A} = \tfrac{1}{5}\partial_{xx}$ generates an exponentially stable semigroup on $L^2(\Omega)$ for $\Omega = (0,1)$ with decay parameter $\mu = \frac{1}{5}\pi^2$.
\end{proposition}
\begin{proof}
    This is an application of \cite[Thm.~11.2.1]{SeiTW22} to the present problem.
    %The proof can be found in \Cref{app::exp_stability_heat}.
\end{proof}

\paragraph{Setup}
For this problem, we use a NN with 4 hidden layers of 10 neurons each and train it for 3000 epochs. The PDE is enforced by 1000 collocation points in the domain $\timeInt \times \Omega$ and the initial condition is trained with 200 data points.

We perform two experiments, one with hard boundary constraints employing \Cref{thm::errorlimit} and one with soft boundary constraints employing \Cref{thm::errorlimit_iss}. In the former case, we design the network with a one-dimensional mask to enforce Dirichlet boundary condition parameterized by $w_1=1$ and $x_1=\tfrac{1}{2}$ (cf. \Cref{sec::methodology}). In the latter case, we train the network to match the boundary condition by using the loss contribution
\begin{equation*}
    L_\mathrm{space} = \frac{1}{N_\mathrm{space}} \sum_{i=1}^{N_\mathrm{space}} \left( | u(t_i, x=0) |^2 + | u(t_i, x=1)| ^2 \right)
\end{equation*}
with fixed weighting $\rho = 10$.

\paragraph{Numerical results for hard boundary constraints}

The trained network agrees well with the known reference solution, i.e., the deviation of the initial values can be limited from above by $\newdelta_0 = 2.9 \cdot 10^{-3}$ and the average residual over all collocation points is given by $\overline{\newdelta} = 2.3 \cdot 10^{-2}$. As shown in \Cref{fig::heat_train_diff_error}, knowing that the semigroup is exponentially decaying improves the error significantly. It outperforms the estimator based on the contraction growth bound by a factor of $2.25$ at the end of the considered time interval. 

\begin{figure}[bt]
    \begin{tikzpicture}
        \pgfplotsset{every tick label/.append style={font=\scriptsize}}
        \pgfplotsset{every axis/.append style={font=\scriptsize}}
        \begin{axis}[ymode=log, xmin=0, xmax=0.5, ymin=0.00001, xlabel={Time $t$}, ylabel=Error, height=5cm, width=9.2cm,grid=both,
            legend cell align=left,
            legend pos= south east,
            legend columns=4,
            legend style={draw=none, nodes={font=\scriptsize}}]
        \addplot[no marks, ErrorPred, dashed, line width=\lineWidth] table [x=t, y=E_init, col sep=comma] {Figures/heat_equation/exp_run_tanh_test_data_t_error_over_time.csv}; \addlegendentry{$\Einit^\mathrm{exp}$};
        \addplot[no marks, ErrorPred, dotted, line width=\lineWidth] table [x=t, y=E_PI, col sep=comma] {Figures/heat_equation/exp_run_tanh_test_data_t_error_over_time.csv}; \addlegendentry{$\Eeq^\mathrm{exp}$};
        \addplot[no marks, ErrorPred, line width=\lineWidth] table [x=t, y=E_tot, col sep=comma] {Figures/heat_equation/exp_run_tanh_test_data_t_error_over_time.csv}; \addlegendentry{$\Einit^\mathrm{exp} + \Eeq^\mathrm{exp}$};
        \addlegendimage{empty legend}
        \addlegendentry{}
        \addplot[no marks, ErrorRef2, dashed, line width=\lineWidth] table [x=t, y=E_init, col sep=comma] {Figures/heat_equation/con_run_tanh_test_data_t_error_over_time.csv};\addlegendentry{$\Einit^\mathrm{con}$};
        \addplot[no marks, ErrorRef2, dotted, line width=\lineWidth] table [x=t, y=E_PI, col sep=comma] {Figures/heat_equation/con_run_tanh_test_data_t_error_over_time.csv};\addlegendentry{$\Eeq^\mathrm{con}$};
        \addplot[no marks, ErrorRef2, line width=\lineWidth] table [x=t, y=E_tot, col sep=comma] {Figures/heat_equation/con_run_tanh_test_data_t_error_over_time.csv};\addlegendentry{$\Einit^\mathrm{con} + \Eeq^\mathrm{con}$};
        \addplot[no marks, ErrorRef, line width=\lineWidth] table [x=t, y=E_ref, col sep=comma] {Figures/heat_equation/con_run_tanh_test_data_t_error_over_time.csv};\addlegendentry{$\Eref$};
        \end{axis}
    \end{tikzpicture}
    \caption{A posteriori error estimator for the heat equation \eqref{eq::heat} on $\timeInt=[0,0.5]$ with hard boundary constraints. The a posteriori error estimator $E_\mathrm{PI} + \Einit$ is depicted with the reference error $\Eref$. The error estimator for the semigroup of the heat equation as exponentially decaying semigroup is shown in dark blue denoted by superscript $\mathrm{exp}$, while the results for the contractive semigroup are shown in red marked with the superscript $\mathrm{con}$. \label{fig::heat_train_diff_error}}
\end{figure}
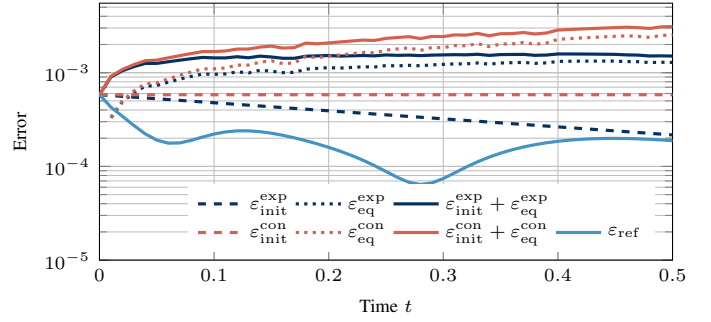

\paragraph{Numerical results for soft boundary constraints}
Instead of training the heat equation with hard boundary constraints, we now use the extended theory with ISS as formulated in \Cref{thm::errorlimit_iss} to demonstrate that the error bound including ISS statements is valid for the heat equation. 
The error prediction is shown in \Cref{fig::heat_error_soft_bcs}. It is visible that the rise in the true error $\Eref$ towards the end of the time interval $[0,0.5]$ is related to an increased error in the boundary condition. This is properly reflected by the contribution $\Ebc$, which is computed using ISS parameters as analyzed in \cite{JacNPS18}. 

\begin{figure}
    \centering
    \begin{tikzpicture}
        \pgfplotsset{every tick label/.append style={font=\scriptsize}}
        \pgfplotsset{every axis/.append style={font=\scriptsize}}
        \begin{axis}[ymode=log, xmin=0, xmax=0.5, ymin=0.000008, xlabel={Time $t$}, ylabel=Error, height=5cm, width=9.2cm,grid=both,
            legend cell align=left,
            legend pos= south east,
            legend columns=2, transpose legend,
            legend style={draw=none, nodes={font=\scriptsize}}]
        \addplot[no marks, ErrorPred, dashed, line width=\lineWidth] table [x=t, y=E_init, col sep=comma] {Figures/heat_equation_soft_bcs/run_tanh_test_data_t_error_over_time.csv}; \addlegendentry{$\Einit$};
        \addplot[no marks, ErrorPred, dotted, line width=\lineWidth] table [x=t, y=E_PI, col sep=comma] {Figures/heat_equation_soft_bcs/run_tanh_test_data_t_error_over_time.csv}; \addlegendentry{$\Eeq$};
        \addplot[no marks, ErrorPred, dash dot, line width=\lineWidth] table [x=t, y=E_bc, col sep=comma] {Figures/heat_equation_soft_bcs/run_tanh_test_data_t_error_over_time.csv}; \addlegendentry{$\Ebc$};
        \addplot[no marks, ErrorPred,  line width=\lineWidth] table [x=t, y=E_tot, col sep=comma] {Figures/heat_equation_soft_bcs/run_tanh_test_data_t_error_over_time.csv}; \addlegendentry{$\Einit + \Eeq + \Ebc$};
        \addplot[no marks, ErrorRef, line width=\lineWidth] table [x=t, y=E_ref, col sep=comma] {Figures/heat_equation_soft_bcs/run_tanh_test_data_t_error_over_time.csv};\addlegendentry{$\Eref$};
        \end{axis}
    \end{tikzpicture}
    \caption{Heat equation with soft boundary constraints: Prediction error estimation including additional error introduced by the boundary error via ISS properties.}
    \label{fig::heat_error_soft_bcs}
\end{figure}
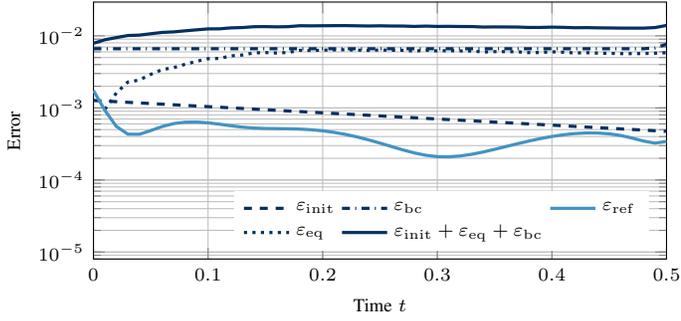

\subsection{Homogeneous advection equation}
We consider the homogeneous advection equation in two space dimensions with one output value

\begin{equation*}
    \begin{aligned}
        \partial_t u(t,\mathbf{x}) =\;& (\tfrac{1}{5}, \tfrac{1}{2}) \cdot \nabla u(t,\mathbf{x})\quad (t,\mathbf{x}) \in \timeInt \times \Omega, \\
        u(0,\mathbf{x}) =\; & \begin{cases}
        \tfrac{1}{2} - \| \mathbf{x} \|_1, \quad\quad\;\; &\mathrm{for }\; \| \mathbf{x} \|_1 \le \tfrac{1}{2}, \\
        0, \quad\quad &\mathrm{else}, 
        \end{cases}  \\
        u(t, -2, \xtwo ) = \;& u(t, 2, \xtwo) \qquad\quad\;\, t\in\timeInt,\; \xtwo \in [-2,2] ,\\
        u(t, \xone, -2) = \;& u(t, \xone, 2) \qquad\quad\;\,  t\in\timeInt,\; \xone \in [-2,2] ,\\
    \end{aligned}
\end{equation*}
for $\Omega \coloneqq (-2,2)^2$ and ${\timeInt} \coloneqq [0,8]$. 

\paragraph{Semigroup properties}
To account for the periodic boundary constraints, we introduce the space
\begin{equation*}
       X \coloneqq \left\lbrace f \in L^2_\mathrm{loc}(\R^2)\; \left|\,
        \begin{aligned} 
        &f(\mathbf{x}) = f(\mathbf{x} +\textstyle  \sum_{i=1,2} 4 k_i\mathbf{e_i}) \\ &\mathrm{for }\; k_1, k_2 \in \Z , \text{a.e.}\ \mathbf{x} \in \R^2
        \end{aligned} \right.\right\rbrace
\end{equation*}
with norm $\|\cdot \|_{L^2{[-2,2]^2}}$, where $\mathbf{e_i}$ denotes the $i$th cartesian unit vectors in $\R^2$.
The operator 
\begin{equation*}
        \mathcal{A}\colon D(\mathcal{A}) \to L^2(\Omega), \qquad 
        u \mapsto (\tfrac{1}{5}, \tfrac{1}{2}) \cdot \nabla u.
\end{equation*}
is bounded on $D(\mathcal{A}) =H^1(\Omega)\cap X $ and generates a strongly continuous semigroup on X. The solution to the advection equation is defined by the shift-semigroup 
\begin{equation*}
    \mathcal{S}_\mathrm{shift, \mathbf{C}}(s) f(\mathbf{x}) = f(\mathbf{x} + s\mathbf{C})
\end{equation*}
with $\mathbf{C}=(\tfrac{1}{5}, \tfrac{1}{2})^T$, which generates an isometry on $D(\mathcal{A})\subseteq L^2_\mathrm{loc}(\R^2)$ \cite[Ch. I, 4.15]{EngN00}. In particular, $\mathcal{S}_\mathrm{shift, \mathbf{C}}$ is a $0$-contraction semigroup.

\paragraph{Setup}
We set up a NN with 8 hidden layers of 40 neurons each and train it for 10000 epochs with $\kappa=0.3$. For the physics-informed contribution 10000 collocation points in the domain $\tilde{\timeInt} \times \Omega$, with $\tilde{\timeInt} \coloneqq [0,4]$, are used. The initial condition is learned on training points located on a grid of $201 \times 201$ for $t=0$. The periodic boundary condition are enforced according to \Cref*{sec::methodology}.

\paragraph{Numerical results}

In \Cref{fig::results_multidim_transport_error}, it is clearly visible that the prediction of the error lies in the same magnitude as the actual error and is mostly dominated by the error on the initial condition for the trained time range $\tilde{\timeInt} =[0,4]$. Even for an extended time domain ${\timeInt} = [0,8]$, for which the network has not been trained, the error estimator is a true upper bound on the prediction error. Since the estimator does not depend on NNs as a method or the quality of the approximation, this is an expected result.

The true error is overestimated increasingly over time but remains at the same magnitude of $10^0$. The lower part of \Cref{fig::results_multidim_transport_error} illustrates that leaving the training domain increases both the error and the overestimation of the error. 

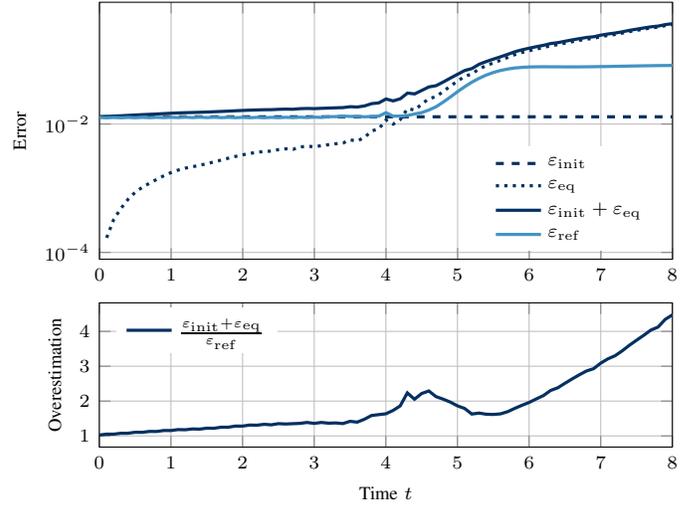
\begin{figure}
    \centering
    \begin{tikzpicture}
        \pgfplotsset{every tick label/.append style={font=\scriptsize}}
        \pgfplotsset{every axis/.append style={font=\scriptsize}}
        \begin{axis}[ymode=log, xmin=0, xmax=8, ylabel=Error, height=5cm, width=9.2cm,grid=both,
            legend cell align=left,
            legend pos= south east,
            legend style={draw=none},
            legend style={draw=none, nodes={font=\scriptsize}},
            name=axis1]
        \addplot[no marks, ErrorPred, dashed, line width=\lineWidth] table [x=t, y=E_init, col sep=comma] {Figures/transport_equation_2dim/run_tanh_input_eval__error_over_time.csv};
        \addplot[no marks, ErrorPred, dotted, line width=\lineWidth] table [x=t, y=E_PI, col sep=comma] {Figures/transport_equation_2dim/run_tanh_input_eval__error_over_time.csv};
        \addplot[no marks, ErrorPred,  line width=\lineWidth] table [x=t, y=E_tot, col sep=comma] {Figures/transport_equation_2dim/run_tanh_input_eval__error_over_time.csv};
        \addplot[no marks, ErrorRef, line width=\lineWidth] table [x=t, y=E_ref, col sep=comma] {Figures/transport_equation_2dim/run_tanh_input_eval__error_over_time.csv};
        \legend{$\Einit$,$\Eeq$ ,$\Einit+ \Eeq$,$\Eref$}
        
        \end{axis}
        \begin{axis}[xmin=0, xmax=8,xlabel={Time $t$}, ylabel={Overestimation}, height=3.5cm, width=9.2cm, grid=both, yshift=-2.5cm,
            legend cell align=left,
            legend pos= north west,
            legend style={draw=none},
            legend style={draw=none, nodes={font=\scriptsize}}, name=axis2]
        \addplot[no marks, ErrorPred, line width=\lineWidth] table [x=t, y=E_rel, col sep=comma] {Figures/transport_equation_2dim/run_tanh_input_eval__error_over_time.csv};
        \legend{$\frac{\Einit + \Eeq}{\Eref}$}
        \end{axis}
    \end{tikzpicture}
    \caption{Top: Error estimation for the two dimensional transport equation for ${\timeInt}=[0,8]$. Bottom: The error estimator overestimates the true error only by an acceptable factor even outside the training range $\tilde{\timeInt} =[0,4]$. }
    \label{fig::results_multidim_transport_error}
\end{figure}

\subsection{Navier-Stokes equations} \label{subsec::NSE}

The Navier-Stokes equations (NSEs) are one prominent example of numerical challenges and their solution via PINNs has been investigated previously in \cite{JinCLK21, PutPSZ22, WanTP21}. We consider velocity-pressure formulation of the NSEs
\begin{equation*}
    \begin{aligned}
        \dt \mathbf{u} + \varrho (\mathbf{u} \cdot \nabla) \mathbf{u} &= - \nabla p + \frac{1}{Re} \nabla^2 \mathbf{u} & \mathrm{in }\; \timeInt &\times \Omega,\\
        \nabla \cdot \mathbf{u} &= 0  &\mathrm{in }\; \timeInt &\times \Omega,\\
        %\mathbf{u} &= \mathbf{u_b} &\mathrm{on }\; \timeInt &\times \partial\Omega,\\
        %\mathbf{u} &= \mathbf{u_0} &\mathrm{in }\; \{t=0\}& \times \Omega,
    \end{aligned} \label{eq::nses}
\end{equation*}
%\BUcomment{$\rho$ is also the contribution of the space loss}
%\BUcomment{Da wir $\mathbf{u_b}$ und $\mathbf{u_0}$ nie angeben, könnten wir die prinzipiell hier auch aus den Gleichungen rausnehmen. Das würde das Layout etwas aufhübschen.}
to address a problem for which the analytical solution is known, the Taylor flow \cite{KimM85}. For the Taylor flow, we consider a two dimensional domain $\Omega \coloneqq (0,\pi)^2 $ with associated velocities $\mathbf{u} = (u,v)^T$, pressure $p$,  $Re = 1$ and $\varrho=1$ on the time domain $\timeInt \coloneqq [0,1]$. The analytical solution is then given by \cite{KimM85}
\begin{equation*}
    \begin{aligned}
        u(t, \xone, \xtwo) =& - \mathrm{cos}(\xone) \sin (\xtwo) \mexp^{-2t}, \\
        v(t, \xone, \xtwo) =& \sin(\xone) \cos (\xtwo) \mexp^{-2t}, \\
        p(t, \xone, \xtwo) =& - \frac{1}{4}( \mathrm{cos}(2 \xone) + \cos (2 \xtwo) )\mexp^{-4t} .\\
    \end{aligned}
\end{equation*}

\paragraph{Semigroup properties}
The semigroup properties for the Navier-Stokes equations can be derived by applying the Helmholtz-Leray projection \cite{ChoM93} on \eqref{eq::nses} and investigating the resulting abstract Cauchy problem. The Helmholtz-Leray projection $\mathcal{P}$ is the projection from $L^2(\Omega) $ into the space of divergence-free vector fields $L_\sigma^2(\Omega ) \coloneqq \mathcal{P} L^2(\Omega)$. Using the Stokes operator 
\begin{equation*}
    \mathcal{A} = \mathcal{P} \Delta \quad, \; D(\mathcal{A}) = H^2(\Omega )^n \cap H^{1}_0(\Omega )^n \cap L_\sigma ^2(\Omega ) ,
\end{equation*}
the projected NSEs are
\begin{equation*}
    \begin{aligned}
        \dot{\uu}(t)  = \mathcal{A} \uu (t) - \mathcal{P} (\uu\cdot \nabla)\uu ,
    \end{aligned}
\end{equation*}
wherein the rightmost term is considered an inhomogeneity. The Stokes operator generates a bounded analytic semigroup and hence a strongly continuous semigroup in $L^2(\Omega)$ with $\Omega \subset \mathbb{R}^d$ (here, $d=2$) bounded with a smooth boundary and Dirichlet boundary condition \cite{Gig81}. In consequence of \cite[Thm.~7.7]{Paz89} and \cite{Gig81}, the operator $\mathcal{P}\Delta$ generates a 0-contraction semigroup on $L^2_\sigma(\Omega)$.

\paragraph{Setup}
We train a PINN with the loss function \eqref{eq::loss}, $\kappa = 0.7$, $\rho = 100$ and  
\begin{equation*}
    L_\mathrm{space} = \sum_{i=1}^{N_\mathrm{eq}} \| \nabla \cdot \mathbf{u} \|^2 .
\end{equation*}
Since the PINN does not automatically satisfy the restriction that the target function must lie in the space of divergence-free functions, we hereby achieve that the violations made by our NN approximation are negligible in comparison to the other approximation errors. Due to this modification in the strategy, our estimator is no longer guaranteed to be a rigorous upper bound.
We use hard boundary conditions as discussed in \cref{subsec::hard_bc} with centers $\xonezero = \xtwozero = \pi/2$ and widths $\wone = \wtwo = \pi$. The contribution $\boldsymbol{\chi}(t, \mathbf{x})$ contains the target boundary condition. Its components are given by
\begin{equation*}
    \begin{aligned} 
        \chi_{1,2}(t, \mathbf{x}) = & \left(-\tfrac{x_{1,2}^2}{\pi^2}+1 \right)\left(\mp \sin(x_{2,1})\mexp^{-2t} \right) \\
        & + \left(-\tfrac{(x_{1,2}-\pi)^2}{\pi^2}+1 \right)\left(\pm \sin(x_{2,1}) \mexp^{-2t} \right).
    \end{aligned} 
\end{equation*}
Here, we want to emphasize that the pressure $p$ (or, to be more precise, its spatial derivative) and the non-linear terms are considered only as an additional contribution to the residual (c.f.~\Cref{rem::inhomogeneous}). Especially the pressure contribution to the residual is justified since it is not governed by a distinct evolution equation. For the sake of completeness, it is included in the initial condition, but the pressure is not restricted by any boundary condition. 

\begin{remark}
    By using methodology as proposed in \cite{KasM20}, one could enforce divergence-free fields in NNs by construction. Since this method relies on Fourier transformation and, therefore, would introduce an extensive bias to a sin/cos-based solution as desired, we decide to approximate the divergence-freeness as a soft constraint.  
\end{remark}

We train a PINN with 10 hidden layers with 80 neurons in each layer over 20000 epochs with 2000 collocation points. The initial condition is trained on a grid of $31 \times 31$ points over the spatial domain.

\paragraph{Numerical results}

The trained network agrees well with the initial condition deviating by $\newdelta_0 = 1.2 \cdot 10^{-3}$, the temporal evolution is also reflected well such that $\overline{\newdelta}= 5 \cdot 10^{-3}$. As visible in \Cref{fig::nse_error}, the actual error decays slightly over time, which is not reflected by the error estimator. Here, the error estimator overestimates the actual error by approximately two magnitudes.

\begin{figure}
    \centering
    \begin{tikzpicture}
        \pgfplotsset{every tick label/.append style={font=\scriptsize}}
        \pgfplotsset{every axis/.append style={font=\scriptsize}}
        \begin{axis}[xmin=0, xmax=1,xlabel={Time $t$}, ymode=log, ylabel=Error, height=5cm, width=9.2cm, grid=both,
            legend cell align=left,
            legend columns = 2,
            legend style={at={(0.95,0.7)}, draw=none, nodes={font=\scriptsize}},
            anchor=west]
        \addplot[no marks, ErrorPred, dashed, line width=\lineWidth] table [x=t, y=E_init, col sep=comma] {Figures/nse/run_tanh_taylor_t_error_over_time.csv};
        \addlegendentry{$\Einit$};
        \addplot[no marks, ErrorPred, dotted, line width=\lineWidth] table [x=t, y=E_PI, col sep=comma] {Figures/nse/run_tanh_taylor_t_error_over_time.csv};
        \addlegendentry{$\Eeq$};
        \addplot[no marks, ErrorPred, dash dot, line width=\lineWidth] table [x=t, y=E_tot, col sep=comma] {Figures/nse/run_tanh_taylor_t_error_over_time.csv};
        \addlegendentry{$\Einit+\Eeq$};
        \addplot[no marks, ErrorRef, line width=\lineWidth] table [x=t, y=E_ref, col sep=comma] {Figures/nse/run_tanh_taylor_t_error_over_time.csv};
        \addlegendentry{$\Eref$};
        \end{axis}
    \end{tikzpicture}
    \caption{Predicted error for the PINN simulating the Taylor flow goverened by the NSE. For comparison, the reference error is computed by numerical integration of the difference between the analytical solution and the PINN prediction.}
    \label{fig::nse_error}
\end{figure}
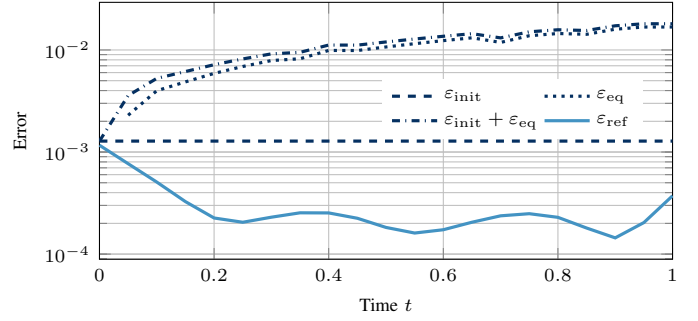

\subsection{Klein-Gordon equation with homogeneous Neumann boundary} \label{subsec::KG}

As a final example, we consider the Klein-Gordon equation with Neumann boundary condition, which is relevant to problems in quantum field theory \cite{Pol01}. The one-dimensional problem can be formulated as 
\begin{equation}
	\label{eq::kg}
    \partial_{tt} u = \partial_{xx} u - \frac{1}{4} u  
\end{equation}
in the domain $\timeInt \times \Omega = [0, 0.2] \times (0,1)$. The initial values are given as 
\begin{equation*}
    \begin{aligned}
        u(0,x) =& \cos(2 \pi x) ,\\
        \partial_t u(0,x) =& \tfrac{1}{2} \cos( 4 \pi x ) ,
    \end{aligned}
\end{equation*}
and we equip~\eqref{eq::kg} with homogeneous Neumann boundary conditions, i.e.
\begin{equation*}
    \partial_x u (t, x\in\{0,1\}) =\;0 .
\end{equation*}
The analytical solution is  
\begin{equation*}
    u(t, x) = \cos(2 \pi x) \cos \left(a_2 t\right) + 
    \tfrac{1}{2 \cdot a_4} \, \cos(4 \pi x) \sin\left( a_4 t\right) 
\end{equation*}
with $a_n = \sqrt{(n \pi)^2 +\tfrac{1}{4}}$; cf.~\cite{Pol01}.

\paragraph{Semigroup properties}
The BIVP for the Klein-Gordon equation can be written as a first order system of the form
\begin{equation*}
        \partial_t \begin{pmatrix}
            u \\ u_t 
        \end{pmatrix}
        = \begin{pmatrix}
            0 & \mathcal{I} \\ \Delta_{xx} - \tfrac{1}{4} & 0 
        \end{pmatrix}
        \begin{pmatrix}
            u \\ u_t 
        \end{pmatrix} .
\end{equation*}

According to \cite{KeyW05} the Neumann Laplacian generates a C0-cosine operator function on $L^2(\Omega)$ and applying results \cite[Lem.~4.48, Rem.~4.49]{Mug14}, the operator 
\begin{equation*}
    \tilde{\mathcal{A}} =  \begin{pmatrix}
        0 & \mathcal{I} \\
        \mathcal{A} + \mathcal{B} & 0
    \end{pmatrix} 
\end{equation*}
for a bounded linear operator $\mathcal{B} \colon D(\mathcal{A}) \rightarrow L^2(\Omega)$ generates a strongly continuous semigroup.

Since $\mathcal{B} = -\tfrac{1}{4} \mathcal{I}$ with $\mathcal{I}$ denoting the identity operator, $\mathcal{B}$ suffices the conditions for \cite[Remark 4.49]{Mug14}. Hence, the operator $\tilde{\mathcal{A}}$ generates a strongly continuous semigroup on $X = L^2(\Omega) \times L^2(\Omega)$ for the domain $D(\tilde{\mathcal{A}}) = H^2(\Omega) \times L^2(\Omega)$. 

We use the results \cite[Rem.~4.49]{Mug14} to derive the necessary parameters for error estimation. According to these results, the resolvent set of the new operator $\rho(\tilde{\mathcal{A}})$ is given by all $\lambda \in \C$ whenever $\lambda^2 \in \rho(\mathcal{A}+ \mathcal{B})$. This then transfers to the spectrum, such that $\lambda \in \sigma(\tilde{\mathcal{A}})$ whenever $\lambda ^2 \in \sigma(\mathcal{A}+ \mathcal{B})$.

The spectrum of the Neumann Laplacian consists of the point spectrum only and is given by 
\begin{equation*}
    \sigma(\Delta_{[0,1]}^N ) = \left\lbrace -k \pi \mid k\in \N_0 \right\rbrace
\end{equation*}
so that the spectrum of $\tilde{\mathcal{A}}$ can be derived to be 
\begin{equation*}
    \sigma(\tilde{\mathcal{A}}) = \left\lbrace \sqrt{-k \pi - 0.25} \mid k\in \N_0 \right\rbrace.
\end{equation*}
Hence, the largest real part of the eigenvalues is given by $\lambda_\mathrm{max} = 0$ for $ k=0$, which we use to apply \Cref{lem::errorlimit_strongly} utilizing the spectral mapping theorem \cite[Thm. 2.4]{Paz89}.

If we consider the Klein-Gordon-equation on the more restrictive space $X = H^2(\Omega) \times L^2(\Omega)$ with norm
\begin{equation}\label{eq::newnorm_kg}
    \| (u, v) \| = \left( \| \nabla u\|^2_{L^2(\Omega)} + 0.25 \|u\|^2_{L^2(\Omega)} + \| v\|^2_{L^2(\Omega)}\right)^{1/2},
\end{equation}
then one can show that the generated semigroup is a 0-contraction semigroup analogous to the proof of \cite[Ch. 7.4, Thm. 5]{Eva10} with minor modifications. 

\paragraph{Numerical determination of M}
This Klein-Gordon example on $X = L^2(\Omega) \times L^2(\Omega)$ shows nicely the technical difficulties arising when the parameter $M$ is required for a complete understanding of the growth behaviour of the weak solution. 
To retrieve a reasonable estimate for $M$, we use the inequality
\begin{equation*} 
    M \ge \frac{\|\mathcal{S}(t) \mathbf{v}\|}{\mexp^{\omega t} \|\mathbf{v}\| }.
\end{equation*} 
Here, we use input data $\mathbf{v}$, which is Gaussian noise with means in the range $[0,20] \times [0,20]$ and variances in the range $[0.05, 0.5] \times [0.05, 0.5]$, and compute $\|\mathcal{S}(t) \mathbf{v}\|$ with the python framework FiPy \cite{GuyWW09}. Thereby we approximate $M \ge 164.43$.

%\paragraph{ISS stability}
%Since ISS can be shown if and only if the semigroup which is generated by $\mathcal{A}$ is exponentially stable \cite{JacNPS18}, the reasoning as applied for the heat equation can not be applied here. However, we penalize deviations from homogeneous Neumann boundary condition heavily so that we neglect the contribution. 

\paragraph{Setup}
We use a NN with 6 hidden layers with 10 neurons each. The initial data is given by 201 equally spaced data points in the range $[0,1]$ and 2500 collocation points are selected equally distributed from $\timeInt \times [0,1]$. The Neumann boundary condition is implemented as soft boundary constraint via 
\begin{equation*}
    L_\mathrm{space} = \frac{1}{N_\mathrm{eq}} \sum_{i=1}^{N_\mathrm{eq}} \left( | u_x(t_{\mathrm{eq},i}, 0) |^2 + | u_x(t_{\mathrm{eq},i}, 1) |^2 \right)
\end{equation*}
with $\rho = 50$. The data-driven loss is more heavily considered by choosing $\kappa=2.5$.
We use multiple optimization algorithms to train this NN properly. Firstly, we run 30000 epochs of the adam optimizer with learning rate $10^{-3}$ to find a suitable starting point for further 50000 epochs of the second order L-BFGS algorithm with learning rate $10^{-2}$.

\paragraph{Numerical results}
The error estimator for the Klein-Gordon equation~\eqref{eq::kg} on $X = L^2(\Omega)\times L^2(\Omega)$ overestimates the true error from the beginning significantly by two orders of magnitude (c.f.~\Cref{fig::error_kleingordon}, top). This improves slightly over time to only one order of magnitude, but is most clearly due to the high value of $M$. The fact that $M \neq 1$ is reflected nicely in \Cref{fig::error_kleingordon} (top), in which the error contribution $\Einit$ is significantly larger than the true error $\Eref$  even for $t=0$.  

When evaluating the NN on $X= H^2(\Omega)\times L^2(\Omega)$ with the norm \eqref{eq::newnorm_kg} we are considering a 0-contraction semigroup, which is properly reflected in \Cref{fig::error_kleingordon} (bottom). In consequence, the error estimator lies in the same order of magnitude as the true error and overestimates the error by less than $10\%$.

\begin{figure}
    \centering
    \begin{tikzpicture}
        \pgfplotsset{every tick label/.append style={font=\scriptsize}}
        \pgfplotsset{every axis/.append style={font=\scriptsize}}
        \begin{axis}[ymode=log, xmin=0, xmax=0.2, ymin=0.006, ylabel=Error, height=5cm, width=9.2cm,grid=both,
            legend cell align=left,
            xtick = {0, 0.05, 0.1, 0.15, 0.2},
            legend pos= south east,
            legend columns=2, transpose legend,
            legend style={draw=none, nodes={font=\scriptsize}}]
        \addplot[no marks, ErrorPred, dashed, line width=\lineWidth] table [x=t, y=E_init, col sep=comma] {Figures/klein_gordon/run_tanh_kleingordon_ref_t_error_over_time.csv}; \addlegendentry{$\Einit$};
        \addplot[no marks, ErrorPred, dotted, line width=\lineWidth] table [x=t, y=E_PI, col sep=comma] {Figures/klein_gordon/run_tanh_kleingordon_ref_t_error_over_time.csv}; \addlegendentry{$\Eeq$};
        \addplot[no marks, ErrorPred,  line width=\lineWidth] table [x=t, y=E_tot, col sep=comma] {Figures/klein_gordon/run_tanh_kleingordon_ref_t_error_over_time.csv}; \addlegendentry{$\Einit + \Eeq$};
        \addplot[no marks, ErrorRef, line width=\lineWidth] table [x=t, y=E_ref, col sep=comma] {Figures/klein_gordon/run_tanh_kleingordon_ref_t_error_over_time.csv};\addlegendentry{$\Eref$};
        \end{axis}
    \end{tikzpicture}

    \begin{tikzpicture}
        \pgfplotsset{every tick label/.append style={font=\scriptsize}}
        \pgfplotsset{every axis/.append style={font=\scriptsize}}
        \begin{axis}[ymode=log, xmin=0, xmax=0.2, ymin=0.006, xlabel={Time $t$}, ylabel=Error, height=5cm, width=9.2cm,grid=both,
            legend cell align=left,
            xtick = {0, 0.05, 0.1, 0.15, 0.2},
            legend pos= south east,
            legend columns=2, transpose legend,
            legend style={draw=none, nodes={font=\scriptsize}}]
        \addplot[no marks, ErrorPred, dashed, line width=\lineWidth] table [x=t, y=E_init, col sep=comma] {Figures/klein_gordon_in_H2L2/run_tanh_kleingordon_ref_t_error_over_time.csv}; \addlegendentry{$\Einit$};
        \addplot[no marks, ErrorPred, dotted, line width=\lineWidth] table [x=t, y=E_PI, col sep=comma] {Figures/klein_gordon_in_H2L2/run_tanh_kleingordon_ref_t_error_over_time.csv}; \addlegendentry{$\Eeq$};
        \addplot[no marks, ErrorPred,  line width=\lineWidth] table [x=t, y=E_tot, col sep=comma] {Figures/klein_gordon_in_H2L2/run_tanh_kleingordon_ref_t_error_over_time.csv}; \addlegendentry{$\Einit + \Eeq$};
        \addplot[no marks, ErrorRef, line width=\lineWidth] table [x=t, y=E_ref, col sep=comma] {Figures/klein_gordon_in_H2L2/run_tanh_kleingordon_ref_t_error_over_time.csv};\addlegendentry{$\Eref$};
        \end{axis}
    \end{tikzpicture}
    \caption{\label{fig::error_kleingordon} Error predicted via the presented methods for the Klein-Gordon equation with homogeneous Neumann boundary condition on $X= L^2(\Omega)\times L^2(\Omega)$ (top) or $X= H^2(\Omega)\times L^2(\Omega)$ (bottom). }
\end{figure}
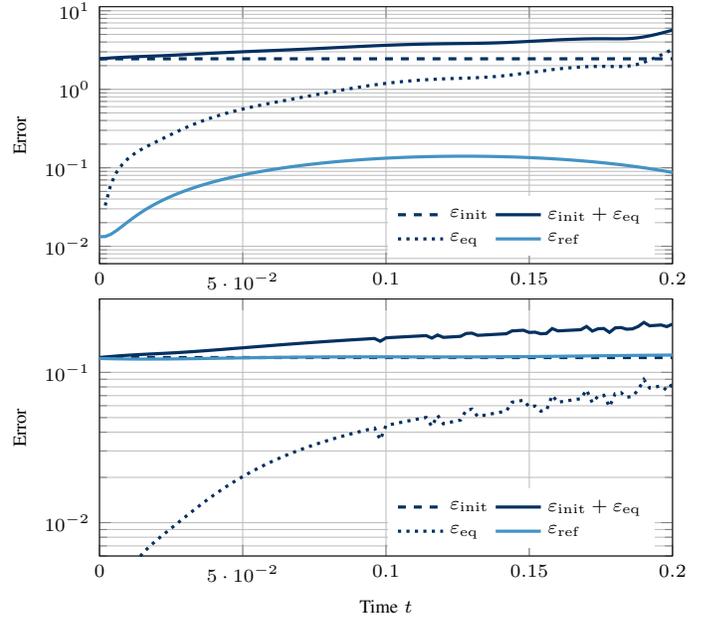

\section{Discussion} \label{sec::discussion}

We derived computable rigorous upper bounds for the prediction error of PINNs and other surrogate modeling techniques that approximate solutions to linear partial differential equations. We illustrated the applicability of the error estimator in several examples. We discuss in the following the most important of the thereby illustrated properties, possibilities, and shortcomings.

Firstly, we demonstrated with the NSEs that the error estimator can also be applied to nonlinear partial differential equations. However, this presumes that the dominant contribution to the time evolution is still linear and leads to a weakening of the guaranteed upper error bound to an approximate error estimation. In the most general case, the previously presented methodology is not directly applicable to nonlinear PDEs, which is subject to further investigation.

Apart from the question of how to thoroughly extend the presented theory to nonlinear PDEs, finding the correct governing semigroup parameters remains a challenge. Even though most real-world problems (such as the Navier-Stokes equations, Maxwell equations, etc.) probably have been discussed extensively as part of many mathematical research areas, the investigations of the parameters necessary for the a posteriori error estimator have been more academic in nature. More precisely, while the existence of these parameters has been proven in various constellations, their explicit calculation is either not directly included in the publications or hidden in lengthy proofs. We have presented for the Klein-Gordon equation how to determine numerically one of the key parameters. This approach could be extended to compute both the growth bound $\omega$ and the scaling factor $M$ by discretizing the problem using finite-element methods and finding the eigenvalues. This problem in general and the latter suggestion need further study.

We have shown for the heat equation that the two different estimators resulting once from the analysis that it is a 0-contractive semigroup and once from the exponential stability analysis, overestimate the true error in different magnitudes. This improvement of the estimator through a more thorough analysis of the semigroup is consistent with the fundamental reason why PINNs are more powerful than purely data-driven NNs: the introduction of a priori knowledge. 

Mostly, except for the heat equation and the Klein-Gordon equation, we used hard boundary constraints. This is not always possible, and when it is, it can be laborious to find a suitable representation. Hence, the extension of the theory using input-to-state stability (ISS) is highly relevant for the practical use of this estimator. Since ISS is restricted to exponentially stable semigroups, a similar analysis for, e.g., the Klein-Gordon equation remains open. 

Beyond the previously mentioned technical limitations and challenges, the methodology does not provide information about a lower limit to the true error, as it would be required for judging the sharpness of the error estimator \cite{Ver96}.

\bibliographystyle{myplain}
\bibliography{literature}

\vfill

\end{document}